\documentclass{elsarticle}\usepackage[]{graphicx}\usepackage{xcolor}
\makeatletter
\def\maxwidth{ %
  \ifdim\Gin@nat@width>\linewidth
    \linewidth
  \else
    \Gin@nat@width
  \fi
}
\makeatother

\definecolor{fgcolor}{rgb}{0.345, 0.345, 0.345}

\usepackage{framed}
\makeatletter
 {\par\unskip\endMakeFramed%
 \at@end@of@kframe}
\makeatother

\definecolor{shadecolor}{rgb}{.97, .97, .97}
\definecolor{messagecolor}{rgb}{0, 0, 0}
\definecolor{warningcolor}{rgb}{1, 0, 1}
\definecolor{errorcolor}{rgb}{1, 0, 0}
\newenvironment{knitrout}{}{} % an empty environment to be redefined in TeX

\usepackage{alltt} 

\usepackage{amssymb}
\usepackage{graphicx}
\usepackage{amsmath}
\usepackage{amsthm}
\usepackage{mathtools}
\usepackage{tikz}
\usepackage{booktabs}
\newtheorem{theorem}{Theorem}

\newcommand{\featdim}{d}
\newcommand{\Nunl}{U}
\newcommand{\Nlab}{L}
\newcommand{\X}{\mathbf{X}  }
\newcommand{\Xe}{\mathbf{X}_e  }
\newcommand{\XeT}{\mathbf{X}_e^{T}}
\newcommand{\ye}{\begin{bmatrix} \mathbf{y}  \\ \mathbf{y}_u \end{bmatrix}}
\newcommand{\G}{\left(\Xe^T \Xe \right)^{-1}}
\newcommand{\missidentity}{\begin{bmatrix} 0 & 0 \\ 0 & \textbf{I }\end{bmatrix}}
\newcommand{\Greg}{\left(\Xe^{T} \Xe + \lambda \missidentity \right)^{-1}}
\newcommand{\Cb}{\mathcal{C}_{\beta}}
\IfFileExists{upquote.sty}{\usepackage{upquote}}{}
\begin{document}

\begin{frontmatter}
\date{}
\title{Robust Semi-supervised Least Squares Classification by Implicit Constraints}

\author[prlab,molepi]{Jesse H. Krijthe\corref{cor1}}
\ead{jkrijthe@gmail.com}

\author[prlab,imagegroup]{Marco Loog}
\ead{m.loog@tudelft.nl}

\cortext[cor1]{Corresponding Author}

\address[prlab]{Pattern Recognition Laboratory, Delft University of Technology, Mekelweg 4, 2628CD Delft, The Netherlands}
\address[molepi]{Department of Molecular Epidemiology, Leiden University Medical Center, Einthovenweg~20, 2333ZC Leiden, The~Netherlands}
\address[imagegroup]{Image Group, University of Copenhagen, Universitetsparken 5, DK-2100, Copenhagen, Denmark}

\begin{abstract}
We introduce the implicitly constrained least squares (ICLS) classifier, a novel semi-supervised version of the least squares classifier. This classifier minimizes the squared loss on the labeled data among the set of parameters implied by all possible labelings of the unlabeled data.
Unlike other discriminative semi-supervised methods, this approach does not introduce explicit additional assumptions into the objective function, but leverages implicit assumptions already present in the choice of the supervised least squares classifier.
This method can be formulated as a quadratic programming problem and its solution can be found using a simple gradient descent procedure. 
We prove that, in a limited 1-dimensional setting, this approach never leads to performance worse than the supervised classifier.
Experimental results show that also in the general multidimensional case performance improvements can be expected, both in terms of the squared loss that is intrinsic to the classifier, as well as in terms of the expected classification error.
\end{abstract}


\end{frontmatter}

\section{Introduction}
We consider the problem of semi-supervised learning of binary classification functions. As in the supervised paradigm, the goal in semi-supervised learning is to construct a classification rule that maps objects in some input space to a target outcome, such that future objects map to correct target outcomes as well as possible. In the supervised paradigm this mapping is learned using a set of $\Nlab$ training objects and their corresponding outputs. In the semi-supervised scenario we are given an additional and often large set of $\Nunl$ unlabeled objects. The challenge of semi-supervised learning is to incorporate this additional information to improve the classification rule.

The goal of this work is to build a semi-supervised version of the least squares classifier that is robust against deterioration in performance meaning that, at least in expectation, its performance is not worse than supervised least squares classification.
While it may seem like an obvious requirement for any semi-supervised method, current approaches to semi-supervised learning do not have this property. 
In fact, performance can significantly degrade as more unlabeled data is added, as has been shown in \cite{Cozman2006,Cozman2003}, among others.
This makes it difficult to apply these methods in practice, especially when there is a small amount of labeled data to identify possible reduction in performance.
A useful property of any semi-supervised learning procedure would therefore be that its performance does not degrade as we add more unlabeled data.
Additionally, many semi-supervised learning procedures are formulated as hard-to-optimize, non-convex objective functions. 
A more satisfactory state of affairs for semi-supervised classification would therefore be methods that are easier to train and that, on average, do not lead to worse classification performance than their supervised alternatives.

We present a novel approach to semi-supervised learning for the least squares classifier that we will refer to as implicitly constrained least squares classification (ICLS). ICLS leverages implicit assumptions present in the supervised least squares classifier to construct a semi-supervised version. This is done by minimizing the supervised loss function subject to the constraint that the solution has to correspond to the solution of the least squares classifier for some labeling of the unlabeled objects.

As this work is specifically concerned with least squares classification, we note several reasons why this is a particularly interesting classifier to study: 
  First of all, the least squares classifier is a discriminative classifier. 
Some have claimed semi-supervised learning without additional assumptions is impossible for discriminative classifiers \cite{Seeger2001,Singh2008}. Our results show this does not strictly hold. 

Secondly, the closed-form solution for the supervised least squares classifier allows us to study its theoretical properties. In particular, in the univariate setting without intercept and assuming perfect knowledge of $P_X$, the distribution of the feature, we show this procedure \emph{never} gives worse performance in terms of the squared loss criterion compared to the supervised least squares classifier. Moreover, using the closed-form solution we can rewrite our semi-supervised approach as a quadratic programming problem, which can be solved through a simple gradient descent with boundary constraints.

Lastly, least squares classification is a useful and adaptable classification technique  allowing for straightforward use of, for instance, regularization, sparsity penalties or kernelization \cite{Hastie2009,Poggio2003,Rifkin2003,Suykens1999,Tibshirani1996}. 
Using these formulations, it has been shown to be competitive with state-of-the-art methods based on loss functions other than the squared loss \cite{Rifkin2003} as well as computationally efficient on large datasets \cite{Bottou2010}.

This work builds on \cite{Krijthe2015} and offers a more complete exposition: we show ICLS can be formulated  as a quadratic programming problem, we extend the experimental results section by including an alternative semi-supervised procedure, adding additional datasets and discussing the ‘peaking’ phenomenon. Moreover, we extend the theoretical result with conditions when one is likely to see improvement of the proposed approach over the supervised classifier.

The main contributions of this paper are
\begin{itemize}
\item A novel convex formulation for robust semi-supervised learning using squared loss (Equation \ref{icls})
\item A proof that this procedure never reduces performance in terms of the squared loss for the 1-dimensional case without intercept (Theorem \ref{theorem:1d})
\item An empirical evaluation of the properties of this classifier (Section \ref{section:empiricalresults})
\end{itemize}

The rest of this paper is organized as follows. 
Section \ref{section:relatedwork} gives an overview of related work on semi-supervised learning. 
Section \ref{section:overview} gives a high level overview of the method while Section \ref{section:method} introduces our semi-supervised version of the least squares classifier in more detail. 
We then derive a quadratic programming formulation and present a simple way to solve this problem through bounded gradient descent. 
Section \ref{section:theoreticalresults} contains a proof of the improvement of the ICLS classifier over the supervised alternative. This proof is specific to classification with a single feature, without including an intercept in the model. For the multivariate case, we present an empirical evaluation of the proposed approach on benchmark datasets in Section \ref{section:empiricalresults} to study its properties. The final sections discuss the results and conclude.

\section{Related Work} 
\label{section:relatedwork}
Many diverse approaches to semi-supervised learning have been proposed   \cite{Chapelle2006,Zhu2009}. While semi-supervised techniques have shown promise in some applications, such as document classification \cite{Nigam2000}, peptide identification \cite{Kall2007} and cancer recurrence prediction \cite{Shi2011}, it has also been observed that these techniques may give performance worse than their supervised counterparts. See for instance \cite{Cozman2006,Cozman2003}, for an analysis of this problem, and \cite{Elworthy1994} for a practical example in part-of-speech tagging.
In these cases, disregarding the unlabeled data would lead to better performance. 

Some \cite{Goldberg2009,Wang2007a} have argued that \emph{agnostic} semi-supervised learning, which \cite{Goldberg2009} defines as semi-supervised learning that is at least no worse than supervised learning, can be achieved by cross-validation on the limited labeled data. 
Agnostic semi-supervised learning follows if we only use semi-supervised methods when their estimated cross-validation error is significantly lower than those of the supervised alternatives.
As the results of  \cite{Goldberg2009} indicate, this criterion may be too conservative: given the small amount of labeled data, a semi-supervised method will only be preferred if the difference in performance is very large. 
If the difference is less distinct, the supervised learner will always be preferred and we potentially ignore useful information from the unlabeled objects. 
Moreover, this cross-validation approach can be computationally demanding. 

\subsection*{Self-Learning}
A simple approach to semi-supervised learning is offered by the self-learning procedure \cite{McLachlan1975} also known as Yarowsky's algorithm \cite{Abney2004,Yarowsky1995} or retagging \cite{Elworthy1994}.
Taking any classifier, we first estimate its parameters on only  the labeled data. 
Using this trained classifier we label the unlabeled objects and add them, or potentially only those we are most confident about, with their predicted labels to the labeled training set. 
The classifier parameters are re-estimated using these labeled objects to get a new classifier. 
One iteratively applies this procedure until the predicted labels of the unlabeled data no longer change.

One of the advantages of this procedure is that it can be applied to any supervised classifier.
It has also shown practical success in some application domains, particularly document classification \cite{Nigam2000,Yarowsky1995}.
Unfortunately, the process of self-training can also lead to severely decreased performance, compared to the supervised solution \cite{Cozman2006,Cozman2003}. 
One can imagine that once an object is incorrectly labeled and added to the training set, its incorrect label may be reinforced, leading the solution away from the optimum. 
Self-learning is closely related to expectation maximization (EM) based approaches \cite{Abney2004}. Indeed, expectation maximization suffers from the same issues as self-learning \cite{Zhu2009}.
In Section~\ref{section:empiricalresults} we compare the proposed approach to self-learning for the least squares classifier.

\subsection*{Additional Assumptions}
Some semi-supervised methods leverage the unlabeled data by introducing assumptions that link properties of the features alone to properties of the label of an object given its features. Commonly used assumptions are the smoothness assumption: objects that are close in the feature space likely share the same label; the cluster assumption: objects in the same cluster share a label; and the low density assumption enforcing that the decision boundary should be in a region of low data density. 

The low-density assumption is used in entropy regularization \cite{Grandvalet2005} as well as for support vector classification in the transductive support vector machine (TSVM)  \cite{Joachims1999} and closely related semi-supervised SVM (S$^3$VM) \cite{Bennett1998,Sindhwani2006}. 
In these approaches an additional term is added to the objective function to push the decision boundary away from regions of high density. 
Several approaches have been put forth to minimize the resulting non-convex objective function, such as the convex concave procedure \cite{Collobert2006} and difference convex programming \cite{Sindhwani2006,Wang2007}.

In all these approaches to semi-supervised learning, a parameter controls the importance of the unlabeled points. When the parameter is correctly set, it is clear, as \cite{Wang2007a} claims, that TSVM is always no worse than supervised SVM. 
It is, however, non-trivial to choose this parameter, given that semi-supervised learning is most interesting in cases where we have limited labeled objects, making a choice using cross-validation very unstable. 
In practice, therefore, TSVM can also lead to performance worse than the supervised support vector machine, as well will also see in Section~\ref{subsection:crossvalidation}.

\subsection*{Safe Semi-supervised Learning}
\cite{Loog2010,Loog2014b} attempt to guard against the possibility of deterioration in performance by not introducing additional assumptions, but instead leveraging implicit assumptions already present in the choice of the supervised classifier.
These assumptions link parameters estimates that depend on labeled data to parameter estimates that rely on all data. 
By exploiting these links, semi-supervised versions of the nearest mean classifier and the linear discriminant are derived. 
Because these links are unique to each classifier, the approach does not generalize directly to other classifiers. 
The method presented here is similar in spirit, but unlike \cite{Loog2010,Loog2014b}, no explicit equations have to be formulated to link parameter estimates using only labeled data to parameter estimates based on all data. Moreover, our approach allows for theoretical analysis of the non-deterioration of the performance of the procedure.

Aside from the work by \cite{Loog2010,Loog2014b}, another attempt to construct a robust semi-supervised version of a supervised classifier has been made in \cite{Li2011}, which introduces the safe semi-supervised support vector machine (S$^4$VM). 
This method is an extension of S$^3$VM \cite{Bennett1998} which constructs a set of low-density decision boundaries with the help of the additional unlabeled data, and chooses the decision boundary, which, even in the worst-case, gives the highest gain in performance over the supervised solution. 
If the low-density assumption holds, this procedure provably increases classification accuracy over the supervised solution. 
The main difference with the method considered in this paper, however, is that we make no such additional assumptions. We show that even without these assumptions, safe improvements are possible for the least squares classifier.

\subsection*{Semi-supervised Least Squares}
While least squares classification has been widely used and studied \cite{Hastie2009, Poggio2003, Suykens1999}, little work has been done on applying semi-supervised learning to the least squares classifier specifically. For least squares regression, \cite{Little2002} describe an iterative method for handling missing outcomes that was formally proposed in \cite{Healy1956}. In the case of least squares regression, this method has some computational advantages over discarding the unlabeled data but its solution always coincides with the supervised solution. \cite{Shaffer1991} studied the value of knowing $\mathbb{E}[\mathbf{X}^T\mathbf{X}]$, where $\mathbf{X}$ is the $\Nlab \times \featdim$ design matrix containing the feature values for each observation. If we assume the number of unlabeled data points is large, this is similar to the semi-supervised situation. It is shown that if the size of the parameters is small compared to the noise, the variance of a procedure that plugs in $\mathbb{E}[\mathbf{X}^T\mathbf{X}]$ as the estimate of $\mathbf{X}^T\mathbf{X}$ has a lower variance than supervised least squares regression.  As the size of the parameters increases, this effect reverses. In fact, the paper demonstrates that in this semi-supervised setting no best linear unbiased estimator for the regression coefficients exists. In Section \ref{section:empiricalresults}, we compare our approach to using this plug-in estimate by substituting the matrix $\mathbf{X}^T\mathbf{X}$ by a version based on both labeled and unlabeled data. 
A similar plug-in procedure has been used by \cite{Fan2008} for linear discriminant analysis for dimensionality reduction which is closely related to least squares classification. Here the (normalized) total scatter matrix, which plays a similar role to the $\mathbf{X}^T\mathbf{X}$ matrix in least squares regression is exchanged with the more accurate estimate of the total scatter based on both labeled and unlabeled data.

\section{Implicitly Constrained Least Squares Classification}
\label{section:overview}
Given a limited set of $\Nlab$ labeled objects and a potentially large set of $\Nunl$ unlabeled objects, the goal of implicitly constrained least squares classification is to use the latter to improve the solution of the least squares classifier trained on just the labeled data. We start with a sketch of this approach, before discussing the details.

\begin{figure}[ht] 
\centering
\includegraphics[width=1\textwidth]{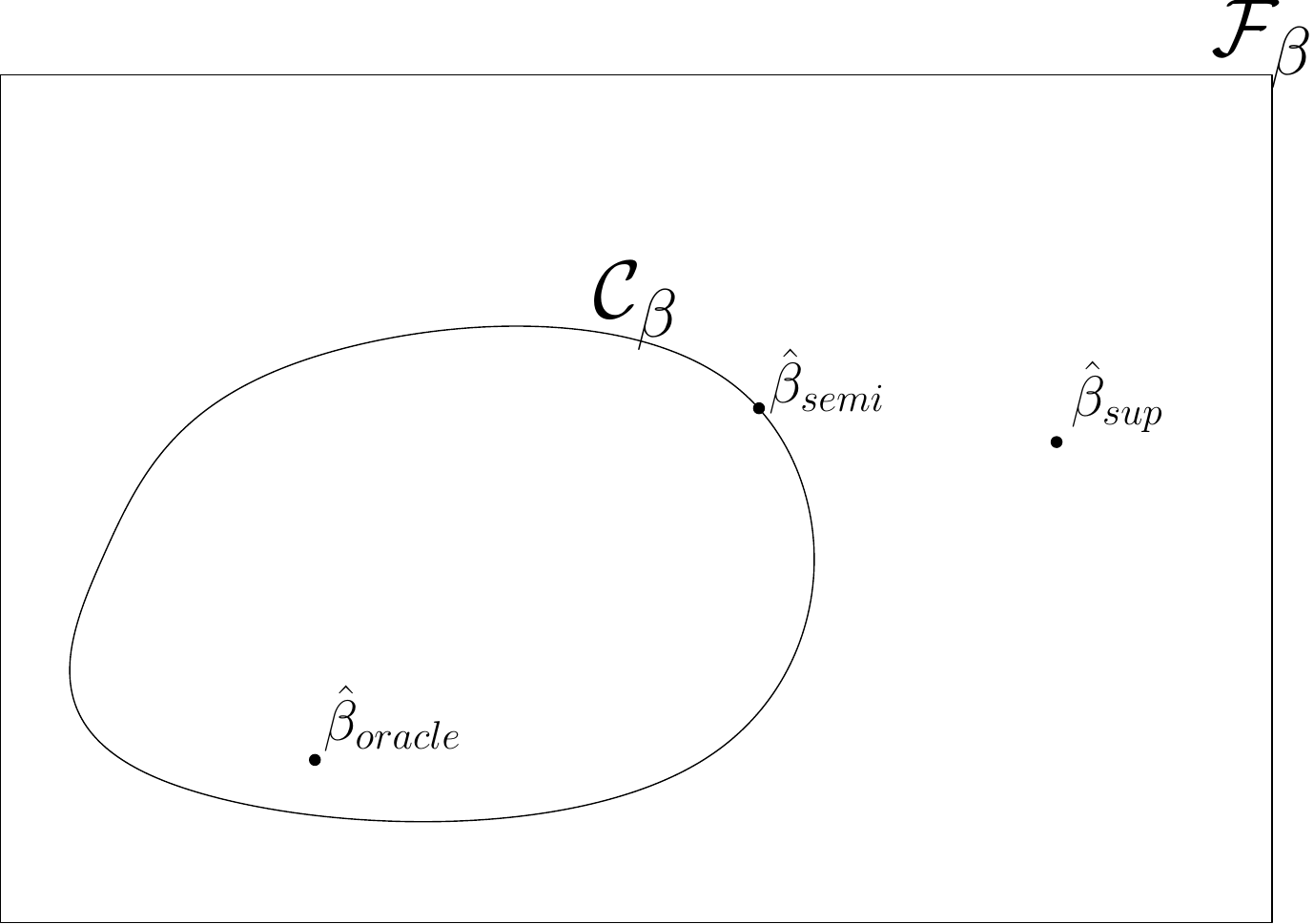}
\caption{A visual representation of implicitly constrained semi-supervised learning. $\mathcal{F}_{\beta}$ is the space of all linear models. $\hat{\beta}_{sup}$ denotes the solution given only a small amount of labeled data. $\Cb$ is the subset of the space which contains all the solutions we get when applying all possible (soft) labelings to the unlabeled data. $\hat{\beta}_{semi}$ is the solution in $\Cb$ that minimizes the loss on the labeled objects. $\hat{\beta}_{oracle}$ is the supervised solution if we would have the labels for all the objects.} \label{fig:constrainedsubset}
\end{figure}

Given the supervised least squares classifier, consider the hypothesis space of all possible parameter vectors, which we will denote as $\mathcal{F}_{\beta}$, see Figure \ref{fig:constrainedsubset}. Given a set of labeled objects, we can determine the supervised parameter vector $\hat{\beta}_{sup}$. Suppose we also have a potentially large number $\Nunl$ of unlabeled objects. Assume that every object has a label, it is merely unknown to us. If these labels were to be revealed, it is clear how the additional objects can improve classification performance: we estimate the least squares classifier using all the data to obtain the parameter vector $\hat{\beta}_{oracle}$. Since this estimate is based on more objects, we expect the parameter estimate to be better. These real labels are unknown, but we can still consider all possible labelings of unlabeled objects, and estimate corresponding parameters based on these imputed labelings. In this way, we get a set of possible parameters for our classifier, which form the set denoted by $\Cb \subset \mathcal{F}_{\beta}$. Clearly one of these labelings corresponds to the real, but unknown, labeling, so one of the parameter estimates in this set corresponds to the solution we would obtain using all the correct labels of both the labeled and unlabeled objects. Because these are the only possible classifiers when the true labels would be revealed, we propose to look within this set $\Cb$ for an improved semi-supervised solution. 

Two issues then remain: how do we choose the best parameters from this set and how do we find these without having to enumerate all possible labelings?

Looking at the first problem, we reiterate that the goal of semi-supervised learning is to find a good classification rule and, therefore, still the obvious way to evaluate this rule is by the loss on the labeled training points. In other words, we choose the classifier from the parameter set that minimizes the squared loss on the labeled points. We will denote this solution by $\hat{\beta}_{semi}$. Note this approach is rather different from other approaches to semi-supervised learning where the loss is adapted by including a term that depends on the unlabeled data points. In our formulation, the loss function is still the regular, supervised loss of our classification procedure.

As for the second issue, after relaxing the constraint that we need hard labels for the data points, we will see that the resulting optimization problem is, in fact, an instantiation of well-studied quadratic programming, which we solve using a simple gradient descent procedure.

\section{Method}
\label{section:method}

\subsection{Supervised Multivariate Least Squares Classification} \label{section:leastsquares}

Least squares classification \cite{Hastie2009,Rifkin2003} is the direct application of well-known ordinary least squares regression to a classification problem. A linear model is assumed and the parameters are minimized under squared loss. Let $\mathbf{X}$ be an $\Nlab \times (\featdim+1)$ design matrix with $\Nlab$ rows containing vectors of length equal to the number of features $\featdim$ plus a constant feature to encode the intercept. Vector $\textbf{y}$ denotes an $\Nlab \times 1$ vector of  class labels. We encode one class as $0$ and the other as $1$.  The multivariate version of the empirical risk function for least squares estimation is given by

\begin{equation} \label{squaredloss}
\hat{R}(\boldsymbol{\beta}) = \frac{1}{\Nlab} \left\|  \mathbf{X} \boldsymbol{\beta}-\mathbf{y} \right\| _2^2 \, .
\end{equation}
The well-known closed-form solution for this problem is found by setting the derivative with respect to $\boldsymbol{\beta}$ equal to $\textbf{0}$ and solving for $\boldsymbol{\beta}$, giving
\begin{equation} \label{olssolution}
\boldsymbol{\hat{\beta}}=\left(\mathbf{X}^T \mathbf{X}\right)^{-1} \mathbf{X}^T \mathbf{y} \, .
\end{equation}
In case $\textbf{X}^T \textbf{X}$ is not invertible (for instance when $L<(\featdim+1)$), a pseudo-inverse is applied. As we will see, the closed form solution to this problem will enable us to formulate our semi-supervised learning approach in terms of a standard quadratic programming problem, which is easy to optimize.

\subsection{Implicitly Constrained Least Squares Classification} \label{section:icls}

In the semi-supervised setting, apart from a design matrix $\textbf{X}$ and target vector $\textbf{y}$, an additional set of measurements $\textbf{X}_u$ of size $\Nunl \times (\featdim+1)$ \emph{without} a corresponding target vector $\textbf{y}_u$ is given. In what follows, $\mathbf{X}_e=\begin{bmatrix} \mathbf{X}^T  & \mathbf{X}_u^T \end{bmatrix}^T$ denotes the extended design matrix which is simply the concatenation of the design matrices of the labeled and unlabeled objects.

In the implicitly constrained approach, we incorporate the additional information from the unlabeled objects by searching within the set of classifiers that can be obtained by all possible labelings $\textbf{y}_u$, for the one classifier that minimizes the \emph{supervised} empirical risk function in Equation~\eqref{squaredloss}. This set, $\Cb$, is formed by the $\boldsymbol{\beta}$s that would follow from training supervised classifiers on all (labeled and unlabeled) objects going through all possible soft labelings for the unlabeled samples, i.e., using all $\textbf{y}_u \in [0,1]^{\Nunl}$. Since these supervised solutions have a closed form, this can be written as
\begin{equation} \label{constrainedregion}
\Cb := \left\{   \boldsymbol{\beta} = \left( {\XeT} {\Xe} \right)^{-1} {\XeT} \ye: \mathbf{y}_u \in [0,1]^{\Nunl} \right\} \, .
\end{equation}
The soft labeling provides both a relaxation for computational reasons as well as a strategy to deal with label uncertainty. We can interpret these fractions as a type of class posterior for the unlabeled objects. 
This constraint set $\Cb$, combined with the supervised loss that we want to optimize in Equation \eqref{squaredloss}, gives the following definition for implicitly constrained semi-supervised least squares classification:
\begin{equation}
\begin{aligned}
&\operatorname*{argmin}_{\boldsymbol{\beta} \in \Cb} & \hat{R}(\boldsymbol{\beta}) \, .
\end{aligned}
\end{equation}
Since $\boldsymbol{\beta}$ is fixed for a particular choice of $\textbf{y}_u$ and has a closed form solution, we can rewrite the minimization problem in terms of $\textbf{y}_u$ instead of $\boldsymbol{\beta}$:
\begin{equation} \label{icls}
\begin{aligned}
& \operatorname*{argmin}_{\mathbf{y}_u \in [0,1]^{\Nunl}} & \frac{1}{\Nlab}  \left\|  \X \left(\XeT \Xe \right)^{-1} \XeT \ye - \mathbf{y} \right\|_2^2 \, . \\ 
\end{aligned}
\end{equation}
The problem defined in Equation \eqref{icls} can be written in a standard quadratic programming  form:
\begin{equation}
\begin{aligned}
& \quad \min_{\mathbf{y}_u} \frac{1}{2} \textbf{y}_u^T  \textbf{Q}  \textbf{y}_u + \textbf{c}^T \textbf{y}_u   & \\
& \text{subject to:}  \quad \begin{bmatrix} \textbf{I}_{\Nunl}  \\ -\textbf{I}_{\Nunl} \end{bmatrix}  \textbf{y}_u \leq \begin{bmatrix} \textbf{1}_{\Nunl}  \\ \textbf{0}_{\Nunl} \end{bmatrix} & \\
\end{aligned}
\end{equation}
where\footnote{The published version of this paper contains a typo in this equation and the two equations that follow. We corrected this error here.}
\begin{equation}
\begin{aligned}
\textbf{Q} = & \frac{2}{\Nlab}  \textbf{X}_u \G \textbf{X}^T \textbf{X} \G \textbf{X}_u^T \,, &\\
\end{aligned} \nonumber
\end{equation}
and
\begin{equation}
\begin{aligned}
\textbf{c} = & \frac{2}{\Nlab} \textbf{X}_u \G \textbf{X}^T \textbf{X} \G \textbf{X}^T \textbf{y} & \\
& - \frac{2}{\Nlab}  \textbf{X}_u \G \textbf{X}^T  \textbf{y} \, . &\\
\end{aligned} \nonumber
\end{equation}
Here, $\textbf{I}_{\Nunl}$ denotes the $\Nunl \times \Nunl$ identity matrix and $\textbf{1}_{\Nunl}$ and $\textbf{0}_{\Nunl}$ denote column vectors of respectively ones and zeros.

Since the matrix \textbf{Q} is a product of a matrix and its transpose, it is guaranteed to be positive semi-definite. The problem is typically not positive definite because there are different labelings that will lead to one and the same minimum objective. 

The quadratic problem defined above can be solved using, for instance, an interior point method. We have found a gradient descent approach to be easier to apply. Taking the derivative with respect to $\textbf{y}_u$ and rearranging the terms we find
\begin{align}
\frac { \partial \hat{R} }{ \partial \textbf{y}_u } = & \ \frac{2}{\Nlab}  \textbf{X}_u \G \textbf{X}^T \textbf{X} \G \textbf{X}^T \textbf{y} \nonumber  \\
& +  \frac{2}{\Nlab}  \textbf{X}_u \G \textbf{X}^T  \textbf{X}  \G  \textbf{X}_u^T \textbf{y}_u  \nonumber \\
& -  \frac{2}{\Nlab}  \textbf{X}_u \G \textbf{X}^T  \textbf{y} \, . \nonumber
\end{align}

Because of its convexity, this problem can be solved efficiently using a quasi-Newton approach that allows for the  $[0,1]$ box bounds, such as L-BFGS-B \cite{Byrd1995}. Solving for $\mathbf{y}_u$ gives a labeling that we can use to construct the semi-supervised classifier using Equation \eqref{olssolution} by considering the imputed labels as the labels for the unlabeled data.

\section{Theoretical Results}
\label{section:theoreticalresults}
We will examine this procedure by considering it in a limited, yet illustrative setting. In this case we will, in fact, prove that our procedure will \emph{never} give a worse least squares estimate than the supervised solution.
Consider the case where we have just one feature $x$, a limited set of labeled instances and assume we know the probability density function of this feature $f_X$ exactly. 
This last assumption is similar to having unlimited unlabeled data and is also considered, for instance, in \cite{Sokolovska2008}. 
We consider a linear model with no intercept: $y = x \beta$ where $y$, without loss of generality, is set as $0$ for one class and $1$ for the other. 
For new data points, estimates $\hat{y}$ can be used to determine the predicted label of an object by using a threshold set at, for instance, $0.5$.

The expected squared loss, or risk, for this model is given by

\begin{equation} \label{eq:trueloss}
R^*(\beta) = \sum_{y \in \{0,1\}}{ \int_{-\infty}^{\infty}(x \beta - y)^2  f_{X,Y}(x,y)  \mathrm{d}x} \,,
\end{equation}
where $f_{X,Y}=P(y|x) f_X(x)$. We will refer to this as the joint density of $X$ and $Y$. Note, however, that this is not strictly a density, since it deals with the joint distribution over a continuous $X$ and a discrete $Y$. The optimal solution $\beta^\ast$ is given by the $\beta$ that minimizes this risk:

\begin{equation} \label{eq:bayesoptimal}
\beta^* = \operatorname*{argmin}_{\beta \in \mathbb{R}} R^*(\beta) \, .
\end{equation}
We will show the following result:
\begin{theorem}
\label{theorem:1d}
Given a linear model in 1D without intercept, $y = x\beta$, and $f_X$ known, the estimate obtained through implicitly constrained least squares always has an equal or lower risk than the supervised solution: $$R^\ast (\hat{\beta}_{semi}) \le R^\ast (\hat{\beta}_{sup}) \, .$$
In particular, given $1$ labeled sample, if $f_{X,Y}$ is continuous in the feature $X$ with bounded second moment and $f_{X,Y}(0,1) > 0$, then $$\mathbb{E}[R^*(\hat{\beta}_{semi})] < \mathbb{E}[R^*(\hat{\beta}_{sup})] \, .$$
\end{theorem}

\begin{proof}

Setting the derivative of \eqref{eq:trueloss} with respect to $\beta$ to $0$ and rearranging we get

\begin{eqnarray}
&\beta & = \left( \int_{-\infty}^{\infty} { x^2 f_X(x) \mathrm{d}x} \right)^{-1} \sum_{y \in \{0,1\}} \int_{-\infty}^{\infty} { x y f_{X,Y}(x,y) \mathrm{d}x } \\
& & =    \left( \int_{-\infty}^{\infty} { x^2 f_X(x) \mathrm{d}x} \right)^{-1}  \int_{-\infty}^{\infty} { x f_X(x) \sum_{y \in \{0,1\}} y P(y|x) \mathrm{d}x} \\
& & =   \left( \int_{-\infty}^{\infty} { x^2 f_X(x) \mathrm{d}x} \right)^{-1}  \int_{-\infty}^{\infty} { x f_X(x) \mathbb{E}[y|x] \mathrm{d}x} \, . \label{eqn:sslsolution}
\end{eqnarray}

In this last equation, since we assume $f_X(x)$ as given, the only unknown is the function $\mathbb{E}[y|x]$, the expectation of the label $y$, given $x$. Now suppose we consider every possible labeling of the unlimited number of unlabeled objects including fractional labels, that is, every possible function where $\mathbb{E}[y|x] \in [0,1]$. Given this restriction on $\mathbb{E}[y|x]$, the second integral in \eqref{eqn:sslsolution} becomes a re-weighted version of the expectation operation over $x$. By changing the choice of $\mathbb{E}[y|x]$ one can vary the value of this integral, but it will always be bounded on an interval on $\mathbb{R}$. It follows that all possible $\beta$'s also form an interval on $\mathbb{R}$, which is the constraint set $\Cb$. The optimal solution has to be in this interval, since it corresponds to a particular but unknown $\mathbb{E}[y|x]$.

\begin{knitrout}
\definecolor{shadecolor}{rgb}{0.969, 0.969, 0.969}\color{fgcolor}\begin{figure}

{\centering \includegraphics[width=\maxwidth]{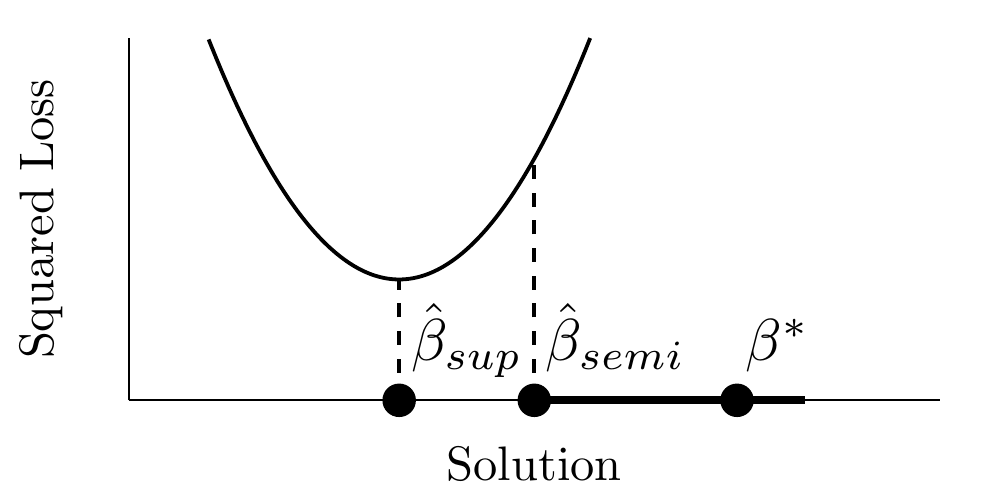} 

}

\caption{An example where implicitly constrained optimization improves performance. The supervised solution $\hat{\beta}_{sup}$ which minimizes the supervised loss (the solid curve), is not part of the interval of allowed solutions. The solution that minimizes this supervised loss within the allowed interval is $\hat{\beta}_{semi}$. This solution is closer to the optimal solution ${\beta}^{\ast}$ than the supervised solution $\hat{\beta}_{sup}$.}\label{fig:constrainedproblem}
\end{figure}

\end{knitrout}

Using the set of labeled data, we can construct a supervised solution $\hat{\beta}_{sup}$ that minimizes the loss on the training set of $\Nlab$ labeled objects (see Figure \ref{fig:constrainedproblem}):
  
\begin{equation} \label{supervisedsolution}
\hat{\beta}_{sup} = \operatorname*{argmin}_{\beta \in \mathbb{R}} \sum_{i=1}^{\Nlab} (x_i \beta - y_i)^2 \, .
\end{equation}

Now, either this solution falls within the constrained region, $\hat{\beta}_{sup} \in \Cb$ or not, $\hat{\beta}_{sup} \notin \Cb$, with different consequences:
  
  \begin{enumerate}
\item If $\hat{\beta}_{sup} \in \Cb$ there is a labeling of the unlabeled points that gives us the same value for $\beta$. Therefore, the solution falls within the allowed region and there is no reason to update our estimate. Therefore $\hat{\beta}_{semi}=\hat{\beta}_{sup}$.
\item Alternatively, if $\hat{\beta}_{sup} \notin  \Cb$, the solution is outside of the constrained region (as shown in Figure \ref{fig:constrainedproblem}): there is no possible labeling of the unlabeled data that will give the same solution as $\hat{\beta}_{sup}$. We then update the $\beta$ to be the $\beta$ within the constrained region that minimizes the loss on the supervised training set. As can be seen from Figure \ref{fig:constrainedproblem}, this will be a point on the boundary of the interval. Note that $\hat{\beta}_{semi}$ is now closer to $\beta^{*}$ than $\hat{\beta}_{sup}$. Since the true loss function $R^*(\beta)$ is convex  and achieves its minimum in the optimal solution, corresponding to the true labeling, the risk of our semi-supervised solution will always be equal to or lower than the loss of the supervised solution.
\end{enumerate}

Thus, the proposed update either improves the estimate of the parameter $\beta$ or it does not change the supervised estimate. In no case will the semi-supervised solution be worse than the supervised solution, in terms of the expected squared loss. This concludes the proof of the first part of the theorem.

The last part of the theorem gives a general condition when, in expectation, our semi-supervised approach will outperform the supervised learner. Because $\hat{\beta}_{semi}$ will never be worse than $\hat{\beta}_{sup}$, to prove this we only need to show that for some observation of a labeled point with positive $f_{X,Y}(x,y)>0$, the estimated $\hat{\beta}_{sup}$ is outside of the interval $\Cb$, in which case $R^*(\hat{\beta}_{semi}) < R^*(\hat{\beta}_{sup})$. 

If we observe an object labeled $1$ with feature value $x$, the corresponding estimate $\hat{\beta}_{sup}=\tfrac{1}{x}$. Since the improvement in loss will only result if this estimate is not in the constrained region, we need to show that

\begin{equation} \label{eq:condition}
P(\tfrac{1}{x} \notin \Cb,y=1)>0 \, .
\end{equation}

To do this, consider the bounds of the interval $\Cb$. These most extreme values are obtained whenever all negative values of $x$ are assigned label $0$ while the positive $x$ get labels $1$, or the other way around. From \eqref{eqn:sslsolution} and writing $\mathbb{E}(X^2)=\int_{-\infty}^{\infty} { x^2 f_X(x) \mathrm{d}x}$ we find the interval is given by

\begin{equation} \label{eq:condition2}
\Cb=\left[ \frac{\int_{-\infty}^{0}{x f_X(x) \mathrm{d}x }}{\mathbb{E}(X^2)},\frac{\int_{0}^{\infty}{x f_X(x)  \mathrm{d}x }}{\mathbb{E}(X^2)} \right] \, .
\end{equation}
Combining this with \eqref{eq:condition}, we get the condition

\begin{equation} \label{eq:condition3}
P \left( \frac{\mathbb{E}(X^2)}{\int_{-\infty}^{0}{x f_X(x)  \mathrm{d}x }} < x < 0 \vee 0 < x < \frac{\mathbb{E}(X^2)}{\int_{0}^{\infty}{x f_X(x)  \mathrm{d}x }},y=1 \right) > 0 \, .
\end{equation}
Since $f_{X,Y}$ is assumed to be continuous, $\mathbb{E}[X^2]>0$, and the lower bound in this equation is always smaller than $0$, while the upper bound is always larger than $0$. The assumption of the continuity of $f_{X,Y}$ ensures that  \eqref{eq:condition3} holds whenever $f_{X,Y}(0,1)>0$.  The property $f_{X,Y}(0,1)>0$ is satisfied by many distributions of the data. The result, therefore, indicates, that in the case of $1$ labeled sample improvement is not only possible, but will occur in many cases. When we have multiple labeled examples, this effect will likely become smaller. This makes sense: the more labeled data we have to estimate the parameter, the smaller the impact of the unlabeled objects will be. \hfill
\end{proof}

\section{Empirical Results} 
\label{section:empiricalresults}
To study the properties of the proposed semi-supervised approach to least squares classification, we compare how this approach fares against supervised least squares classification without the constraints. 

For comparison we include two alternative semi-supervised approaches and an oracle solution: 
  \paragraph{Self-Learning} Using a simple procedure proposed by \cite{McLachlan1975}, among others, the supervised least squares classifier is updated iteratively by using its class predictions on the unlabeled objects as the labels for the unlabeled objects in the next iteration. This is done until convergence.

\paragraph{Updated Second Moment Least Squares (USM)} In this approach we replace the second moment matrix $\mathbf{X}^T \mathbf{X}$ with an appropriately scaled matrix $\XeT \Xe$ similar to the estimator studied in \cite{Shaffer1991}:
  $$
  \boldsymbol{\hat{\beta}}_{USM} = \left( \tfrac{\Nlab}{\Nlab+\Nunl} \XeT \Xe \right)^{-1} \mathbf{X}^T \mathbf{y}
$$
  where $\Xe$ and $\mathbf{y}$ are centered. This centering ensures that results do not depend on the particular encoding of the labels used. We will refer to this as updated second moment least squares (USM) classification. 

\paragraph{Oracle} The performance of the least squares classifier if all unlabeled objects were labeled as well. This serves as the unattainable upper bound on the performance of any semi-supervised learner.

A description of the datasets used for our experiments is given in Table \ref{table:datasets}. We use datasets from both the UCI repository \cite{Lichman2013} and from the benchmark datasets proposed by \cite{Chapelle2006}. While the benchmark datasets proposed in \cite{Chapelle2006} are useful, in our experience, the results on these datasets are very homogeneous because of the similarity in their dimensionality and their low Bayes errors. The UCI datasets are more diverse both in terms of the number of objects and features as well as the nature of the underlying problems. Taken together, this collection allows us to investigate the properties of our approach for a wide range of problems.
All the code used to run the experiments is available from the first author's website.

% latex table generated in R 3.2.1 by xtable 1.7-4 package
% Thu Nov 12 10:25:41 2015
\begin{table}[ht]
\centering
\caption{Description of the datasets used in the experiments. PCA99 refers to the number of principal components required to retain at least 99\% of the variance. Majority refers to the proportion of the number of objects from the largest class} 
\label{table:datasets}

\begin{tabular}{lrrrrr}

  \toprule
Dataset & Objects & Features & PCA99 & Majority & Source \\ 
  \midrule
\textsc{Haberman} & 306 &   3 &   3 & 0.74 & \cite{Lichman2013} \\ 
  \textsc{Ionosphere} & 351 &  33 &  30 & 0.64 & \cite{Lichman2013} \\ 
  \textsc{Parkinsons} & 195 &  22 &  12 & 0.75 & \cite{Lichman2013} \\ 
  \textsc{Diabetes} & 768 &   8 &   8 & 0.65 & \cite{Lichman2013} \\ 
  \textsc{Sonar} & 208 &  60 &  43 & 0.53 & \cite{Lichman2013} \\ 
  \textsc{SPECT} & 267 &  22 &  21 & 0.79 & \cite{Lichman2013} \\ 
  \textsc{SPECTF} & 267 &  44 &  37 & 0.79 & \cite{Lichman2013} \\ 
  \textsc{Transfusion} & 748 &   4 &   3 & 0.76 & \cite{Lichman2013} \\ 
  \textsc{WDBC} & 569 &  30 &  17 & 0.63 & \cite{Lichman2013} \\ 
  \textsc{Mammography} & 961 &   9 &   9 & 0.54 & \cite{Lichman2013} \\ 
  \textsc{Digit1} & 1500 & 241 & 221 & 0.51 & \cite{Chapelle2006} \\ 
  \textsc{USPS} & 1500 & 241 & 183 & 0.80 & \cite{Chapelle2006} \\ 
  \textsc{COIL2} & 1500 & 241 & 114 & 0.50 & \cite{Chapelle2006} \\ 
  \textsc{BCI} & 400 & 117 &  45 & 0.50 & \cite{Chapelle2006} \\ 
  \textsc{g241c} & 1500 & 241 & 235 & 0.50 & \cite{Chapelle2006} \\ 
  \textsc{g241d} & 1500 & 241 & 235 & 0.50 & \cite{Chapelle2006} \\ 
   \bottomrule
\end{tabular}
\end{table}

\subsection{Peaking Behaviour in Semi-supervised Least Squares}
With fewer than $\featdim$ samples, the \emph{supervised} least squares classifier that utilizes a pseudo-inverse is known to exhibit a peaking phenomenon, as described in \cite{Opper1996,Raudys1998}: Starting from a single observation, expected classification errors generally decrease as we add more data before errors increase again to reach a maximum approximately when the number of features is equal to the number of observations. This phenomenon can also be observed in the semi-supervised setting. Figures~\ref{fig:peaking} and \ref{fig:peakingzoom} show learning curves of the methods considered here, using $10$ labeled training objects and an increasing number of unlabeled objects. Performance is evaluated on objects that were not in the labeled or unlabeled set. The Oracle classifier indicates the mean error when we do have the labels for the unlabeled objects and therefore corresponds to the peaking phenomenon in the supervised case. In the supervised case, several proposals have been done to ameliorate this peaking behaviour, such as feature selection, regularization, removing objects, injecting noise in the features, or adding redundant features \cite{Skurichina1999}. The semi-supervised learners suffer from the same peaking phenomenon, except that unlike the Oracle, USM and ICLS do not fully recover from the initial increase in classification error.

We have no full explanation for the observed peaking behaviour in the semi-supervised setting. Even in the supervised setting the behaviour remains elusive. The two observation we do make are: 1. that the peak occurs at the same location for both the supervised and semi-supervised scenarios, which is likely due to the dependence of all methods on the inverse of $\XeT \Xe$ and 2. that the subspace defined by the input data is the defining characteristic for the location of the peak.

This peaking behaviour is not the primary topic of this work and in the remainder we will restrict our attention to the case where there are enough labeled objects such that the matrix $\mathbf{X}^T \mathbf{X}$ is invertible. 

\begin{knitrout}
\definecolor{shadecolor}{rgb}{0.969, 0.969, 0.969}\color{fgcolor}\begin{figure}
\includegraphics[width=\maxwidth]{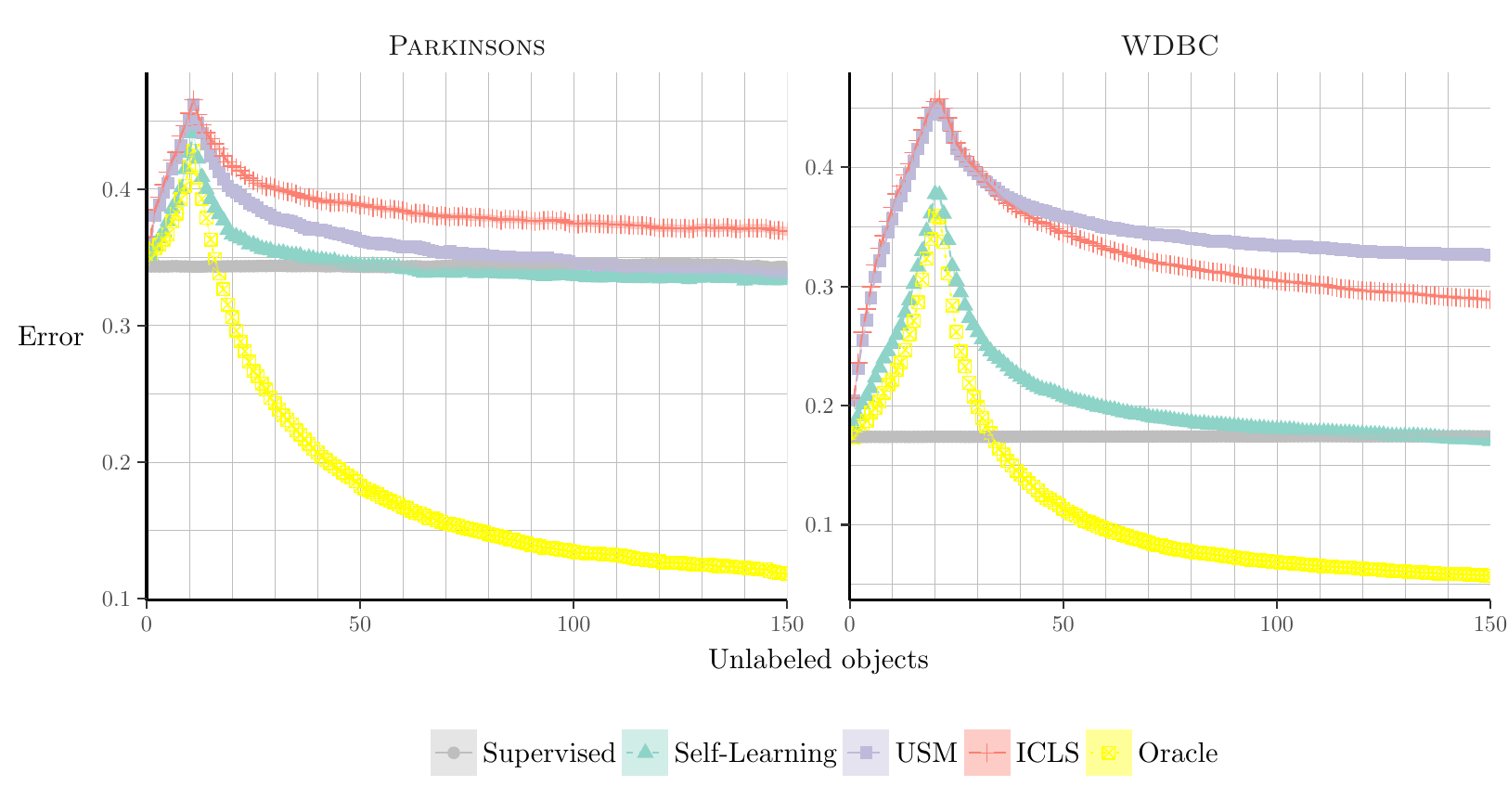} \caption[Peaking phenomenon in Semi-supervised Least Squares Classification]{Peaking phenomenon in Semi-supervised Least Squares Classification. The lines indicate the Mean classification error for $\Nlab=\max(\featdim+5,20)$ and $1000$ repeats. The shaded areas indicate $+/-$ the standard error of the mean, which are so small in this case, they are barely perceptible.}\label{fig:peaking}
\end{figure}

\end{knitrout}

\begin{knitrout}
\definecolor{shadecolor}{rgb}{0.969, 0.969, 0.969}\color{fgcolor}\begin{figure}
\includegraphics[width=\maxwidth]{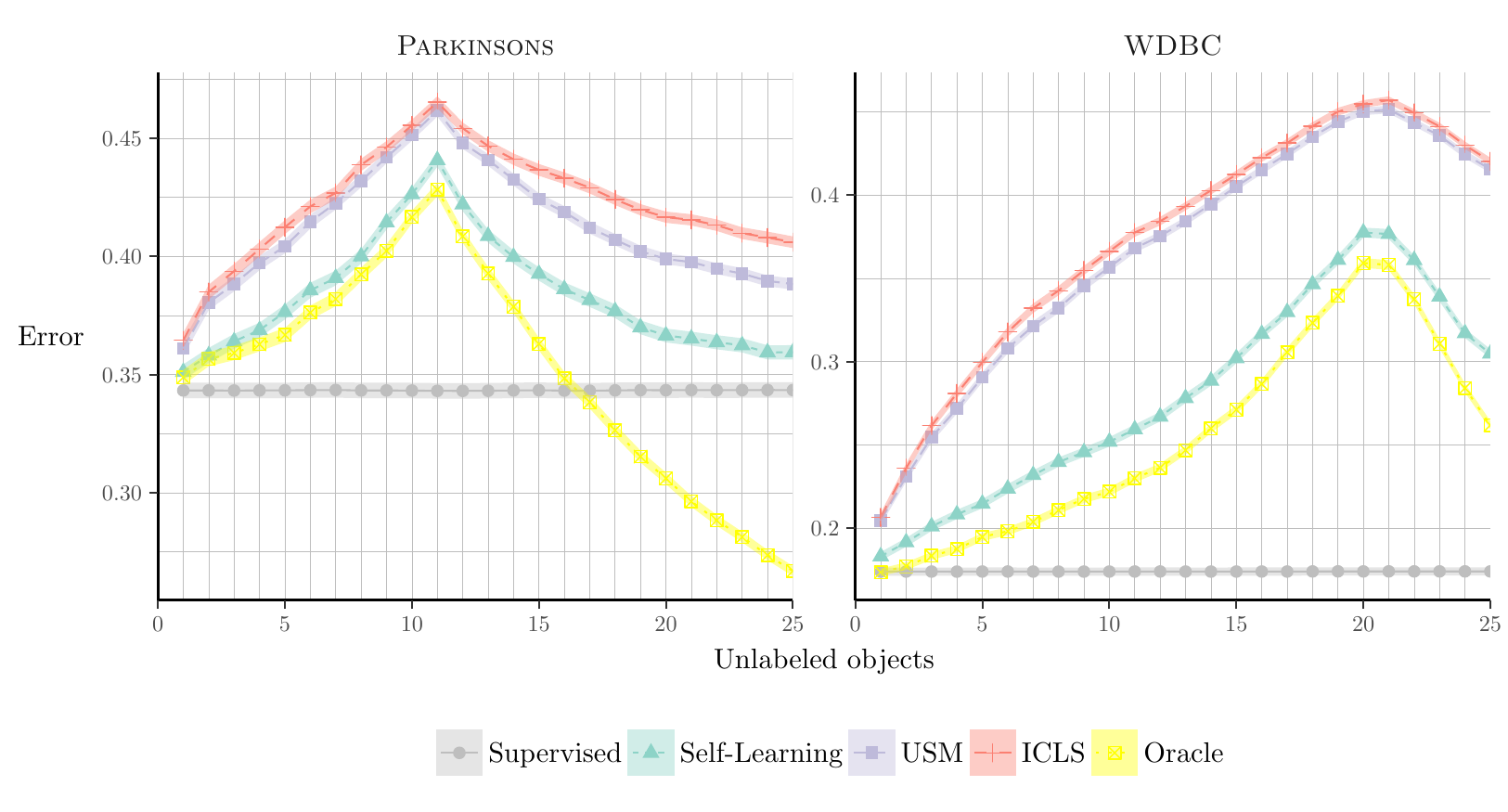} \caption{Zoomed in version of Figure~\ref{fig:peaking}}\label{fig:peakingzoom}
\end{figure}

\end{knitrout}

\subsection{Comparison of Learning Curves}
We study the behavior of the expected classification error of the ICLS procedure for different sizes of the unlabeled set. This statistic has two desired properties. First of all it should never be higher than the expected classification error of the supervised solution, which is based on only the labeled data. Secondly, the expected classification error should not increase as we add more unlabeled data. A semi-supervised classifier that has both these properties can be used safely, since adding unlabeled data and continuing to add more unlabeled data will never decrease performance, on average. 

Experiments were conducted as follows. For each dataset, $\Nlab$ labeled points were randomly chosen, where we make sure to sample at least 1 object from each of the two classes. Since the peaking phenomenon described in the previous section is not main topic of this work, we avoid this situation by considering the setting in which the labeled design matrix is of full rank, which we ensure by setting $\Nlab=\featdim+5$, the dimensionality of the dataset plus five observations. For all datasets we ensure a minimum of $\Nlab=20$ labeled objects.

Next, we create unlabeled subsets of increasing size $\Nunl=[2,4,8,...,1024]$ by randomly selecting points from the original dataset without replacement. The classifiers are trained using these subsets and the classification performance is evaluated on the remaining objects. Since the test set decreases in size as the number of unlabeled objects increases, the standard error slightly increases with the number of unlabeled objects.

The results of these experiments are shown in Figure \ref{fig:errorcurves}. We report the mean classification error as well as the standard error of this mean. As can be seen from the tight confidence bands, this offers an accurate estimate of the expected classification error.

This procedure of sampling labeled and unlabeled points is repeated $1000$ times and the average classification error (Figure \ref{fig:errorcurves}) and squared loss (Figure \ref{fig:losscurves}) on the test set is determined. The latter is done to evaluate whether the approach is effective in increasing generalization performance in terms of the loss used in estimating the classifier. This is the same loss that we consider in Theorem \ref{theorem:1d}. Even though in applications the ultimate goal may typically be classification performance, this allows us to study whether problems occur because of the optimization itself, or because of the link between the surrogate loss used and the classification error.

\begin{knitrout}
\definecolor{shadecolor}{rgb}{0.969, 0.969, 0.969}\color{fgcolor}\begin{figure}
\includegraphics[width=\maxwidth]{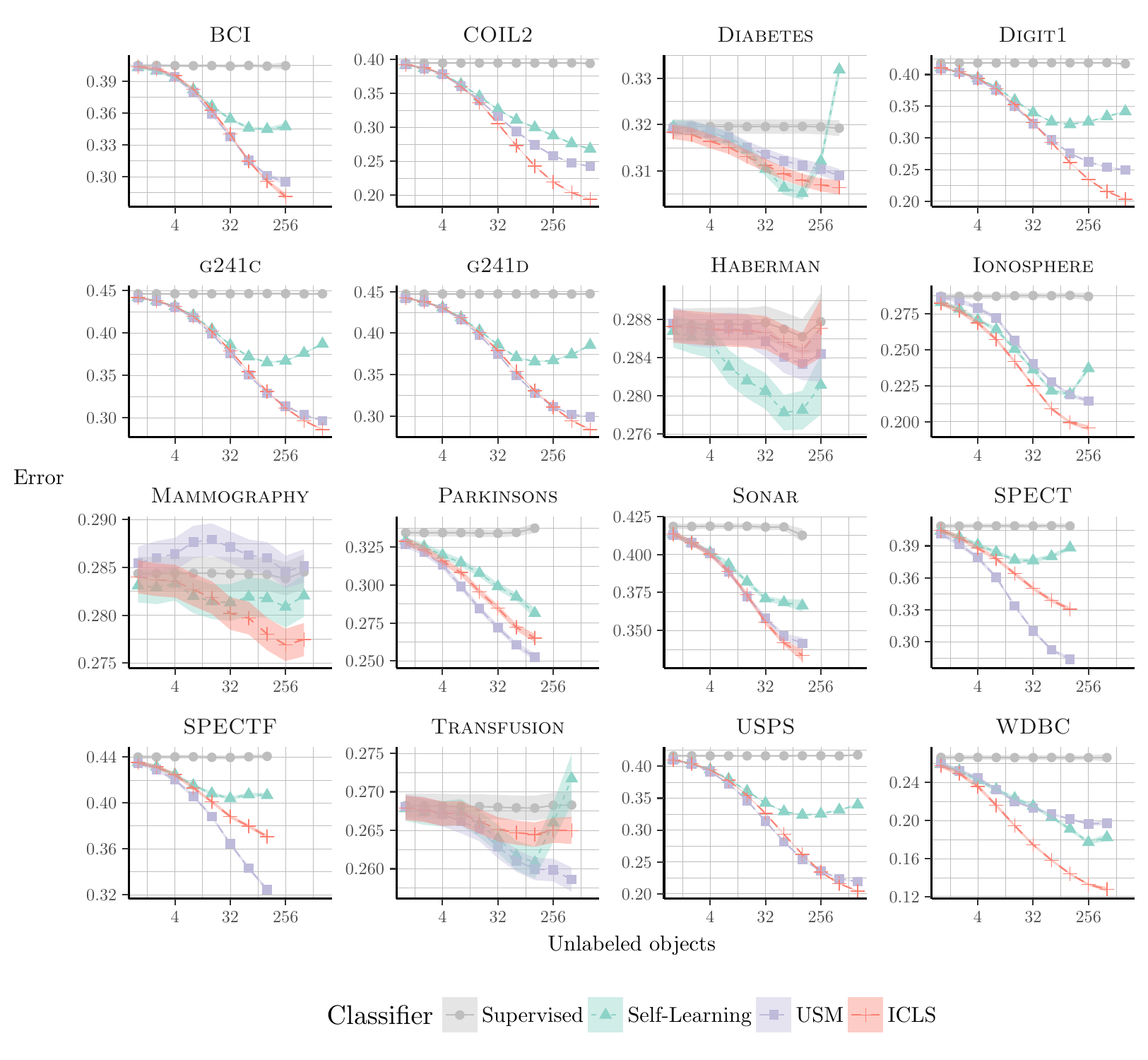} \caption[Mean classification error for $\Nlab=\max(\featdim+5,20)$ and $1000$ repeats]{Mean classification error for $\Nlab=\max(\featdim+5,20)$ and $1000$ repeats. The shaded areas indicate $+/-$ the standard error of the mean, which are so small in some cases, they are barely perceptible.}\label{fig:errorcurves}
\end{figure}

\end{knitrout}

\begin{knitrout}
\definecolor{shadecolor}{rgb}{0.969, 0.969, 0.969}\color{fgcolor}\begin{figure}
\includegraphics[width=\maxwidth]{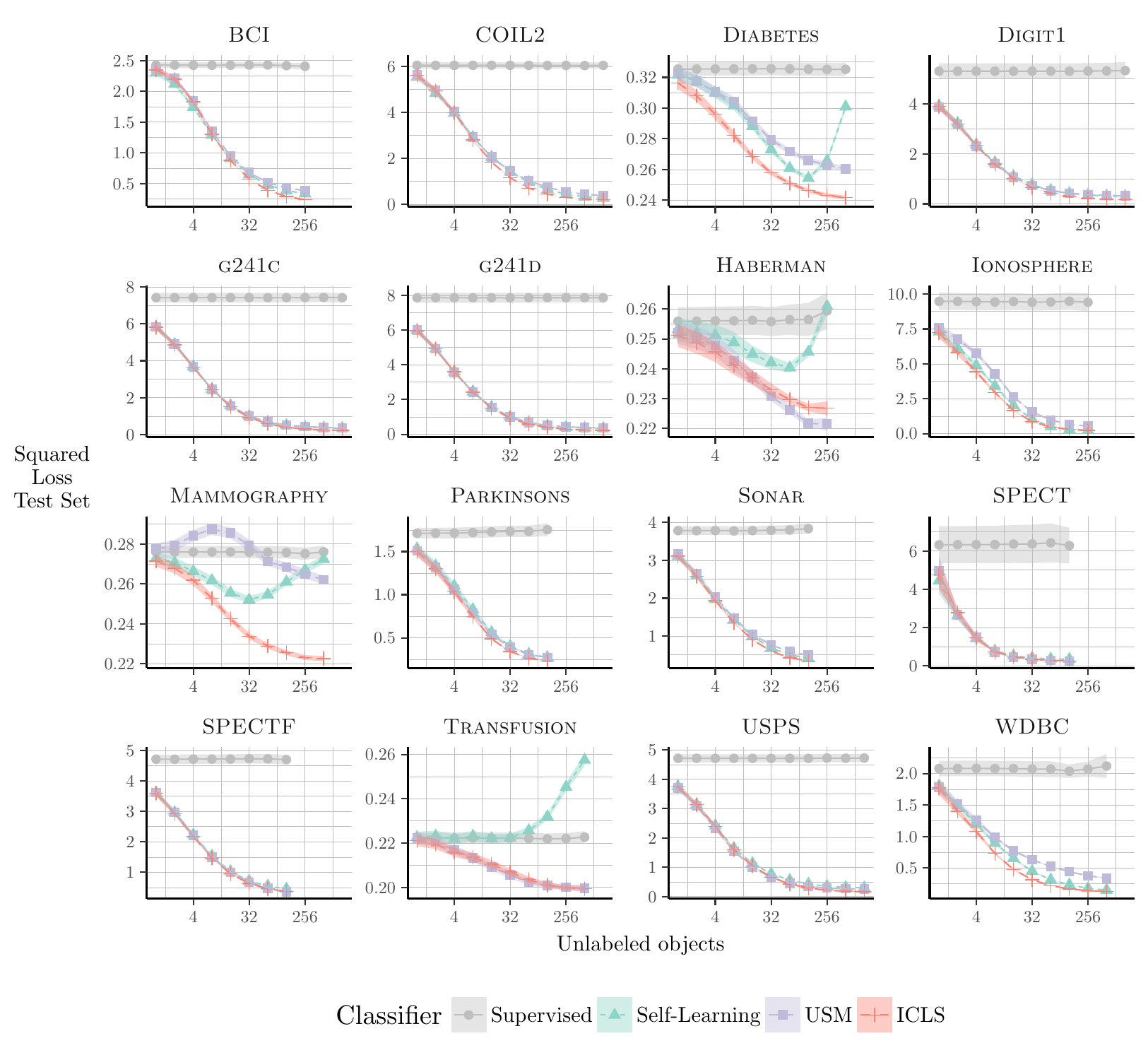} \caption[Mean squared loss on the test set for $\Nlab=\max(\featdim+5,20)$ and $1000$ repeats]{Mean squared loss on the test set for $\Nlab=\max(\featdim+5,20)$ and $1000$ repeats. The shaded areas indicate $+/-$ the standard error of the mean, which are so small in some cases, they are barely perceptible.}\label{fig:losscurves}
\end{figure}

\end{knitrout}

We find that, generally, the ICLS procedure has monotonically decreasing error curves as the number of unlabeled samples increases, unlike self-learning. On the \textsc{Diabetes} and \textsc{Transfusion} datasets, the performance of self-learning becomes worse than the supervised solution when more unlabeled data is added, while the ICLS classifier again exhibits a monotonic decrease of the average error rate. The USM classifier performs well on most datasets except for the \textsc{Mammography} dataset, where both in terms of average error rates and squared loss, performance is worse than the supervised classifier.

When we compare the error curves and the loss curves, the non-monotonically decreasing losses for the self-learner correspond to increased errors. In general, however, similar losses for different classifiers can give rise to different behaviours in terms of error rates.

\subsection{Benchmark performance} \label{subsection:crossvalidation}
We now consider the performance of these classifiers in a cross-validation setting. The experiment is set up as follows. For each dataset, the objects are randomly divided into $10$ folds. We iteratively go through the folds using $1$ fold as validation set, and the other $9$ as the training set. From this training set, we then randomly select $\Nlab=\max(\featdim+5,20)$ labeled objects, as in the previous experiment, and use the rest as unlabeled data. After predicting labels for the validation set for each fold, the classification error is then determined by comparing the predicted labels to the real labels. This is repeated $100$ times, while randomly assigning objects to folds in each iteration.

The cross-validation procedure used here is slightly different from that described in \cite{Chapelle2006}, to make it more closely relate to the cross-validation procedure that is usually employed in supervised learning. More specifically, our procedure ensures the validation sets are independent (non-overlapping), such that, after going over all the folds, each object is in the validation set only once. This is different from the procedure in \cite{Chapelle2006}, were the authors ensure the \emph{labeled} sets are non-overlapping. We have not found a qualitative difference in the error rates, however, when using the procedure proposed in \cite{Chapelle2006}. The advantage of the procedure employed here is that every object gets a single predicted label, allowing for the direct comparison of predictions of different classifiers.

\begin{table}[t]
\caption{Results for the Least Squares Classifier. Average 10-fold cross-validation error and (between parentheses) number of times the error of the semi-supervised classifier is higher than the supervised error for $100$ repeats. Indicated in $\mathbf{bold}$ is which semi-supervised classifier has lowest average error. A Wilcoxon signed rank test at $0.01$ significance level is done to determine whether a semi-supervised classifier is significantly worse than the supervised classifier, indicated by \underline{underlined} values.} \label{table:cvresults}
\centering
\small
\begin{tabular}{llllll}
  \toprule
% latex table generated in R 3.3.2 by xtable 1.8-2 package
% Fri Jan 27 21:09:19 2017
Dataset & Supervised & Self-Learning & USM & ICLS & Oracle \\ 
  \midrule
\textsc{Haberman} & 0.29 & \textbf{0.28 (33)} & 0.28 (42) & 0.29 (24) & 0.26 (11) \\ 
  \textsc{Ionosphere} & 0.29 & 0.24 (1) & 0.22 (1) & \textbf{0.19 (0)} & 0.13 (0) \\ 
  \textsc{Parkinsons} & 0.34 & 0.29 (5) & \textbf{0.25 (3)} & 0.26 (1) & 0.12 (0) \\ 
  \textsc{Diabetes} & 0.32 & \underline{0.34 (83)} & 0.31 (31) & \textbf{0.31 (7)} & 0.23 (0) \\ 
  \textsc{Sonar} & 0.42 & 0.37 (5) & 0.34 (3) & \textbf{0.33 (1)} & 0.25 (0) \\ 
  \textsc{SPECT} & 0.41 & 0.39 (28) & \textbf{0.28 (0)} & 0.33 (1) & 0.18 (0) \\ 
  \textsc{SPECTF} & 0.43 & 0.40 (14) & \textbf{0.31 (0)} & 0.36 (2) & 0.23 (0) \\ 
  \textsc{Transfusion} & 0.27 & \underline{0.28 (63)} & \textbf{0.26 (30)} & 0.27 (25) & 0.23 (2) \\ 
  \textsc{WDBC} & 0.27 & 0.18 (0) & 0.20 (2) & \textbf{0.13 (0)} & 0.04 (0) \\ 
  \textsc{Mammography} & 0.28 & 0.28 (28) & 0.28 (54) & \textbf{0.27 (14)} & 0.20 (0) \\ 
  \textsc{Digit1} & 0.42 & 0.34 (0) & 0.25 (0) & \textbf{0.20 (0)} & 0.06 (0) \\ 
  \textsc{USPS} & 0.42 & 0.34 (0) & 0.22 (0) & \textbf{0.20 (0)} & 0.09 (0) \\ 
  \textsc{COIL2} & 0.39 & 0.27 (0) & 0.24 (0) & \textbf{0.19 (0)} & 0.10 (0) \\ 
  \textsc{BCI} & 0.41 & 0.35 (1) & 0.30 (0) & \textbf{0.28 (0)} & 0.16 (0) \\ 
  \textsc{g241c} & 0.45 & 0.39 (0) & 0.30 (0) & \textbf{0.29 (0)} & 0.14 (0) \\ 
  \textsc{g241d} & 0.45 & 0.39 (0) & 0.30 (0) & \textbf{0.29 (0)} & 0.13 (0) \\ 
   \bottomrule

\end{tabular}
\end{table}

The results shown in Table \ref{table:cvresults} tell a similar story to those in the previous experiment. Most importantly for the purposes of this paper, ICLS, in general, offers solutions that give at least no higher expected classification error than the supervised procedure. 
On many of these datasets, the self-learning approach seems to share this property. However, if we look at for how many of the cross-validation repeats the ICLS and self-learning give lower error than the supervised solution, there is a clear difference. The self-learning solution gives a higher error on more of the repeats than ICLS, for all of the datasets.

The results also show that unlabeled information is of use. Particularly on the last six datasets, ICLS and USM offers large improvement in classification accuracy over the supervised solution. The differences in performance between ICLS and self-learning can also be quite substantial, where ICLS outperforms self-learning on most of the datasets.
USM performs well on many of the datasets, especially when we consider how simple and computationally efficient this procedure is.

\begin{table}[t]
\caption{Results for the Support Vector Classifier. Average 10-fold cross-validation error and (between parentheses) number of times the error of the semi-supervised classifier is higher than the supervised error for $100$ repeats. Indicated in $\mathbf{bold}$ is which semi-supervised classifier has lowest average error. A Wilcoxon signed rank test at $0.01$ significance level is done to determine whether a semi-supervised classifier is significantly worse than the supervised classifier, indicated by \underline{underlined} values.} \label{table:cvresults-svm}
\centering
\begin{tabular}{llllll}
\toprule
% latex table generated in R 3.3.2 by xtable 1.8-2 package
% Fri Jan 27 21:09:19 2017
Dataset & Supervised & Self-Learning & TSVM & Oracle \\ 
  \midrule
\textsc{Haberman} & 0.29 & \textbf{0.29 (34)} & \underline{0.32 (92)} & 0.26 (8) \\ 
  \textsc{Ionosphere} & 0.17 & \underline{0.18 (81)} & 0.17 (51) & 0.11 (0) \\ 
  \textsc{Parkinsons} & 0.22 & \textbf{0.22 (32)} & 0.22 (60) & 0.14 (0) \\ 
  \textsc{Diabetes} & 0.31 & 0.31 (40) & \textbf{0.28 (7)} & 0.23 (0) \\ 
  \textsc{Sonar} & 0.26 & 0.26 (53) & \textbf{0.25 (33)} & 0.25 (25) \\ 
  \textsc{SPECT} & 0.30 & 0.28 (13) & \textbf{0.25 (3)} & 0.18 (0) \\ 
  \textsc{SPECTF} & 0.30 & 0.29 (28) & \textbf{0.28 (29)} & 0.21 (0) \\ 
  \textsc{Transfusion} & 0.27 & 0.27 (59) & \underline{0.29 (96)} & 0.23 (0) \\ 
  \textsc{WDBC} & 0.06 & 0.06 (53) & \textbf{0.05 (30)} & 0.03 (0) \\ 
  \textsc{Mammography} & 0.27 & 0.28 (60) & \textbf{0.25 (3)} & 0.20 (0) \\ 
  \textsc{Digit1} & 0.08 & \underline{0.08 (85)} & \textbf{0.06 (1)} & 0.05 (0) \\ 
  \textsc{USPS} & 0.14 & 0.13 (17) & \textbf{0.12 (5)} & 0.11 (1) \\ 
  \textsc{COIL2} & 0.16 & \underline{0.16 (75)} & \underline{0.19 (100)} & 0.09 (0) \\ 
  \textsc{BCI} & 0.28 & \underline{0.29 (70)} & \underline{0.36 (99)} & 0.17 (0) \\ 
  \textsc{g241c} & 0.22 & \underline{0.23 (87)} & \textbf{0.17 (0)} & 0.16 (0) \\ 
  \textsc{g241d} & 0.23 & \underline{0.24 (90)} & \textbf{0.17 (0)} & 0.16 (0) \\ 
   \bottomrule

\end{tabular}
\end{table}

While we are interested in a semi-supervised procedure that outperforms the supervised least squares classifier, for comparison we repeated the experiment for the (linear) supervised SVM, self-learning applied to the SVM and the Transductive SVM. We used the SVM and TSVM implementations of \cite{Sindhwani2006},  setting the $L_2$ regularization parameter to $\lambda=1$ and the influence parameter of the unlabeled data to $1$, as was also done in \cite{Sindhwani2006}. The experiment is set up in the same way as the one in Table \ref{table:cvresults}. The results are shown in Table~\ref{table:cvresults-svm}.

On many of the datasets, the supervised support vector classifier has a lower error than the supervised least squares classifier, due to the use of a regularization term in the SVM implementation, which we do not include in our analysis and which makes the results difficult to compare directly to the results in Table~\ref{table:cvresults}. Self-learning performs worse compared to the least squares setting, which may be a consequence of the supervised solution already being a decent solution on some of these datasets. The Transductive SVM offers some improvements over the supervised solution. Compared to ICLS, however, the TSVM gives worse performance than the supervised solution on many more datasets and many more repeats, the exact behaviour we attempted to avoid when constructing ICLS.

\section{Discussion}
\subsection*{From Theory to Empirical Results}
The results presented in this paper are rather promising, especially in the light of the negative theoretical performance results presented in the literature \cite{Cozman2006}. The result in Theorem \ref{theorem:1d}, to start with, indicates the proposed procedure is in some way robust against reduction in performance. The strong result of this theorem, stating that performance never gets worse, holds in the 1D case with unlimited unlabeled data and no intercept in the model. A slightly weaker result, that performance does not degrade on average may still hold without these assumptions. This last statement is corroborated by the empirical results showing improvements in averaged squared errors for ICLS throughout.

The results in the previous section also indicate that such improved results hold in terms of the misclassification error, at least on this collection of datasets. These empirical observations are encouraging because we are often interested in misclassification error and not the squared loss that was considered in Theorem \ref{theorem:1d}. Furthermore the experiments were carried out in the multivariate setting with an intercept term using limited unlabeled data, rather than the unlimited unlabeled data setting considered in the theorem. This indicates that minimizing the supervised loss over the subset $\Cb$, leads to a semi-supervised learner with desirable behavior, both theoretically in terms of risk and empirically in terms of classification error.

\subsection*{Robustness}
The method considered in this work is different from most previous work in semi-supervised learning in that it is inherently robust against a decrease in performance. 
The robustness of the method comes from the fact that we do not accept solutions that do not work on the labeled data. The goal of semi-supervised learning is to improve supervised techniques using the additional information inherent in the additional unlabeled objects. Previous approaches have done this by changing the loss function that is being optimized, in particular by introducing an extra term corresponding to assumptions about the unlabeled data. The loss function then becomes a mixture between the supervised objective and an unsupervised objective, which may lead to decreased performance as we observed in Table~\ref{table:cvresults-svm}. If the goal is classification, we propose that the loss function should remain the supervised loss function. The unlabeled objects are merely used to introduce constraints on the possible solutions to this loss function, but do not change its functional form.

\subsection*{Assumptions}
Most other semi-supervised techniques rely on introducing useful assumptions that link information about the distribution of the features $P_X$ to the posterior of the classes $P_{Y|X}$. It has been argued that, for discriminative classifiers, semi-supervised learning is impossible without these additional assumptions about the link between labeled and unlabeled objects \cite{Seeger2001,Singh2008}. ICLS, however, is both a discriminative classifier and no explicit additional assumptions about this link are made.  Any assumptions that are present follow, implicitly, from the choice of squared loss as the loss function and from the chosen hypothesis space.

In fact, additional assumptions may actually be at the root of the problem: clearly if such an additional assumption is correct, a semi-supervised classifier can gain from it, but if the assumption is incorrect, degraded performance may ensue.  What we leverage in our approach are the implicit assumptions that are, in a sense, intrinsic to the supervised least squares classifier. 

One could argue that constraining the solutions to $\Cb$ is an assumption as well. It corresponds to a very weak assumption about the supervised classifier: that it will improve when we add additional labeled data. This is generally assumed in the supervised setting as well. The lack of additional assumptions has another advantage: no additional hyperparameter value needs to be selected that controls the importance of the unlabeled data for the results in Sections \ref{section:theoreticalresults} and \ref{section:empiricalresults} to hold as ICLS acts as a type of data dependent regularization.

Note that the solution provided by self-learning is, by construction, also in the constrained subset $\Cb$. The difference with ICLS is that in ICLS the choice of estimate from $\Cb$ is based on information of the labeled objects only, while self-learning also uses the imputed labels on the unlabeled objects. This may lead to self-deception: if the imputed labels are wrong, a good fit for these wrongly imputed labels does not necessarily lead to an improved $\boldsymbol{\beta}$. In fact, it might lead to worse choices as shown in the results.

\subsection*{Time Complexity}
In terms of the number of features, ICLS scales in the same way as the supervised least squares solution, where the main bottleneck is the calculation of $(\XeT \Xe)^{-1}$. Furthermore, the quadratic programming formulation of ICLS presented in Section \ref{section:method} allows one to use the standard and constantly improving tools from convex optimization to find the ICLS estimate. Unfortunately one has to go from a convex problem with $\featdim+1$ variables in the supervised case to a constrained convex problem with $\Nunl$ variables for ICLS. For very large $\Nunl$, this may not currently be computationally feasible. Further insight in the general nature of the semi-supervised solutions that one obtains can lead to more dedicated and potentially better scalable methods to solve the quadratic programming problem we have to deal with in our approach.

Compared to ICLS, self-learning seems more favorable in terms of computational cost. Self-learning usually converges in a few iterations, where each iteration has at most the cost of one supervised least squares estimation. In our implementations, however, self-learning and ICLS had similar training times (Figure \ref{fig:timecurves}). USM with its simple closed form solution has much lower training times and performs surprisingly well.

\begin{knitrout}
\definecolor{shadecolor}{rgb}{0.969, 0.969, 0.969}\color{fgcolor}\begin{figure}
\includegraphics[width=\maxwidth]{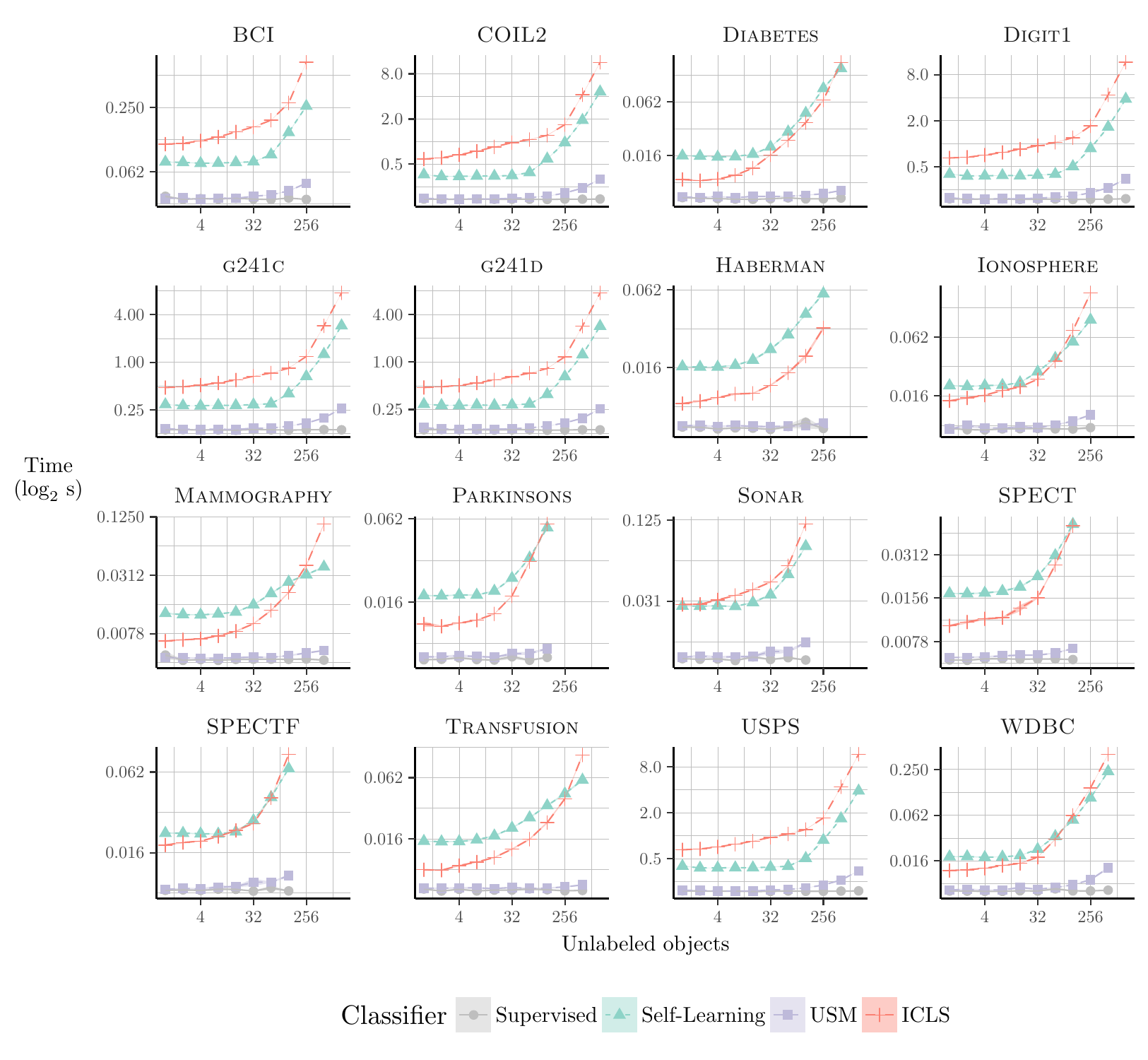} \caption[Average Training Time for $1000$ repeats]{Average Training Time for $1000$ repeats. The shaded areas indicate $+/-$ the standard error of the mean.}\label{fig:timecurves}
\end{figure}

\end{knitrout}

\subsection*{Squared Loss}
Generally, models used in practice do not directly minimize misclassification error. For computational reasons, often convex surrogate losses, such as the one employed here are minimized. It is therefore interesting to look at the performance of a classifier in terms of these surrogate losses \cite{Loog2016a}. We have chosen to restrict ourselves to a particular convex loss and attempted to ensure improvement in terms of this chosen loss function.

When we compare the average squared loss on the test set, ICLS, USM and self-learning often seem to offer similar performance. This is quite unlike the results in, for instance \cite{Loog2010,Loog2014b}, where the self-learner often performed much worse in terms of the loss than an approach based on constraining the solution using unlabeled data. While \cite{Loog2010,Loog2014b} consider a generative classifier, we consider a discriminative classifier, in which case self-learning may be less susceptible to increases in the loss. Self-learning does, however, still increase the loss on some datasets, unlike ICLS.

The peaking phenomenon described in \cite{Opper1996,Raudys1998} is known to occur for squared loss minimization when we increase the number of labeled samples. Here we find it also occurs when we change the number of unlabeled samples. It seems that ICLS and USM are more sensitive to this problem than self-learning. As yet, we do not have any explanation for this behavior. Further improvements to the current approach may start by trying to understand this occurrence of peaking.

\subsection*{Other Losses}
While the results presented in this work are promising for squared loss, an open question is what other classifiers could benefit from the implicitly constrained approach considered here. Using negative log likelihood as a loss function, for instance, also leads to an interesting implicitly constrained semi-supervised classifier, for instance, in linear discriminant analysis \cite{Krijthe2014}. 

In the derivation of ICLS, we made use of the closed-form solution given an imputed labeling to derive a quadratic programming problem in terms of the labels. For many loss functions, closed-form solutions do not exist, which prohibits a straightforward formulation of their implicitly constrained semi-supervised counterparts. Without a supervised closed-form solution one cannot straightaway apply techniques like gradient descent to the parameters as this typically leads to solutions that are outside of the set $\Cb$, even if the loss considered is differentiable.

\subsection*{More Constraints}
In Figure \ref{fig:constrainedsubset}, we illustrate that projecting onto the subset $\Cb$ causes improvement as long as a better solution $\hat{\beta}_{oracle}$ than the supervised solution is within $\Cb$. A smaller $\Cb$ will give a larger improvement, since the semi-supervised solution is going to be closer to $\hat{\beta}_{oracle}$. In the extreme case where only $\hat{\beta}_{oracle}$ forms the subset, this clearly gives a large improvement over supervised learning. It therefore makes sense to think about reducing the size of $\Cb$. In the approach presented in this work, however, to ensure a better solution $\hat{\beta}_{oracle}$ than the supervised solution is always within the constraint set with probability $P(\hat{\beta}_{oracle} \in \Cb)=1$, our choice of $\Cb$ is conservatively large. It contains elements corresponding to all labelings of the unlabeled points, even extremely unlikely ones. 

By excluding unlikely labelings from the subset, the size of $\Cb$ may shrink, while the probability that it includes $\hat{\beta}_{oracle}$ remains high. For instance, one might exclude labelings with class priors that are very unlikely to occur, given the class priors that are observed in the labeled data, a strategy which is also employed in Transductive SVMs where it is necessary for it to converge to meaningful local optima. Changes to $\Cb$ may, therefore, allow for larger improvements in terms of the risk or classification error, while introducing a small chance of deterioration in performance.

\section{Conclusion}
This work introduced a new semi-supervised approach to least squares classification. By implicitly considering all possible labelings of the unlabeled objects and choosing the one that minimizes the loss on the labeled observations, we derived a robust classifier with a simple quadratic programming formulation. For this procedure, in the univariate setting with a linear model without intercept, we can prove it never degrades performance in terms of squared loss (Theorem 1). Experimental results indicate that in expectation this robustness also holds in terms of classification error on real datasets. Hence, semi-supervised learning for least squares classification without additional assumptions can lead to improvements over supervised least squares classification both in theory and in practice.

\section*{Acknowledgement}
Part of this work was funded by project P23 of the Dutch public-private research community COMMIT.

\bibliographystyle{elsarticle-num}
\bibliography{library}

\end{document}